\documentclass[twocolumn]{article}

\usepackage{verbatim}
\usepackage{amssymb,amsmath}
\usepackage{amsthm}
\usepackage[ruled,linesnumbered]{algorithm2e}
\usepackage{verbatim}
\usepackage{wrapfig}
\usepackage[pdftex]{graphicx}
\usepackage{relsize}
\usepackage{bibspacing}
\newcommand{\Xa}{\mathcal{X}}
\newcommand{\Pa}{\mathcal{P}}
\newcommand{\Za}{\mathcal{Z}}
\newcommand{\Wa}{{\bf WALDO\;}}

\DeclareMathOperator{\E}{\mathbb{E}}

\newtheorem{example}{Example}
\usepackage{multirow}
\usepackage{subcaption}
\usepackage{floatrow}
\usepackage{float}
\usepackage{sidecap}
\usepackage{booktabs}
\newtheorem{thm}{Theorem}

\begin{document}

\title{\Large Probabilistic Outlier Detection and Generation}
\author{Stefano Giovanni Rizzo\thanks{Qatar Computing Research Institute} \\
\texttt{strizzo@hbku.edu.qa}
\and Linsey Pang\thanks{Walmart Labs} \\
\texttt{linsey.pang@walmartlabs.com}
\and Yixian Chen\footnotemark[2]\\
\texttt{yixian.chen@walmartlabs.com}
\and  Sanjay Chawla\footnotemark[1]\\
\texttt{schawla@hbku.edu.qa}
}

\date{}

\maketitle

% Copyright Statement
% When submitting your final paper to a SIAM proceedings, it is requested that you include 
% the appropriate copyright in the footer of the paper.  The copyright added should be 
% consistent with the copyright selected on the copyright form submitted with the paper.
% Please note that "20XX" should be changed to the year of the meeting.

% Default Copyright Statement
%\fancyfoot[R]{\scriptsize{Copyright \textcopyright\ 2021 by SIAM\\
%Unauthorized reproduction of this article is prohibited}}

\vspace*{-1cm}
\begin{abstract}
A  new method for outlier detection and generation is introduced by  lifting  data into the space of  probability distributions which are not analytically expressible, but from which samples can be drawn using a neural generator. Given a mixture of unknown latent inlier and outlier distributions, a Wasserstein double autoencoder is used to both detect and generate inliers and outliers. The proposed method, named WALDO (Wasserstein Autoencoder for Learning the Distribution of Outliers), is evaluated on classical data sets including MNIST, CIFAR10 and KDD99 for detection accuracy and robustness. We give an example of outlier detection on a real retail sales data set and an example of outlier generation for simulating intrusion attacks.
%As an application of outlier generation, we demonstrate how our approach can be used to generate realistic  intrusion  attacks that can be used for  simulation exercises to protect network infrastructure.  
However we foresee many application scenarios where WALDO can be used. To the best of our knowledge this is the first work that studies both outlier detection and generation together.
% A case study using dataset from a large e-commerce retailer is used to demonstrate how WALDO can be used to effectively learn and generate anomalies in an automated fashion.
\\
\textbf{Key Words:}  Outlier Detection, Outlier Generation, Wasserstein Distance, Wasserstein Autoencoder 
\end{abstract}

\section{Introduction}

A well known definition of outliers  states, 
``An outlier is an observation that {\it deviates} so much from other observations as to arouse suspicion that it was {\it generated} by a different mechanism~\cite{huber2004robust}.'' Many methods in outlier detection have been inspired by focusing on the deviation aspect of above definition. For example, distance-based techniques define outliers as those data points that are far away from their neighbors; density-based approaches search for outliers in regions of low relative density; the one-class svm method defines outliers as those points that lie outside the tighest  hypersphere containing  most of the points~\cite{chandola2007outlier,charubook}. \\

 In this work we will focus both on the {\it detection} and {\it generation} mechanisms of outliers.  In particular we will assume that data is generated from an unknown and unlabeled mixture of inlier and outlier distributions. We will have 
 access to  samples from  only the unlabeled and the inlier distributions. Our primary objective will be to
 infer the outlier distribution without having recourse to outlier samples. \\

To infer the outlier distribution we will take a probabilistic view of autoencoders which have
been used before for outlier detection~\cite{zhou2017anomaly}. An autoencoder can be seen as a self-mapping from an input space to
itself mediated through a bottleneck - a lower dimensional representation of the data being mapped.
In classical autoencoders, outliers are defined as those data points which have high reconstruction errors.
A probabilistic view is to perceive the self-mapping as one inducing a new probability
distribution on the input space, i.e., an autoencoder maps the original data distribution into a new 
distribution constrained by the bottleneck. \\

%$h$  as inducing a mapping between the data probability distribution $P_X$
%and the push-forward distribution $h\#P_X$ which minimizes
%a measure of divergence between two 
%probability distributions. 

\begin{figure*}
        \begin{subfigure}[b]{0.65\textwidth}
               \includegraphics[width=0.78\linewidth]{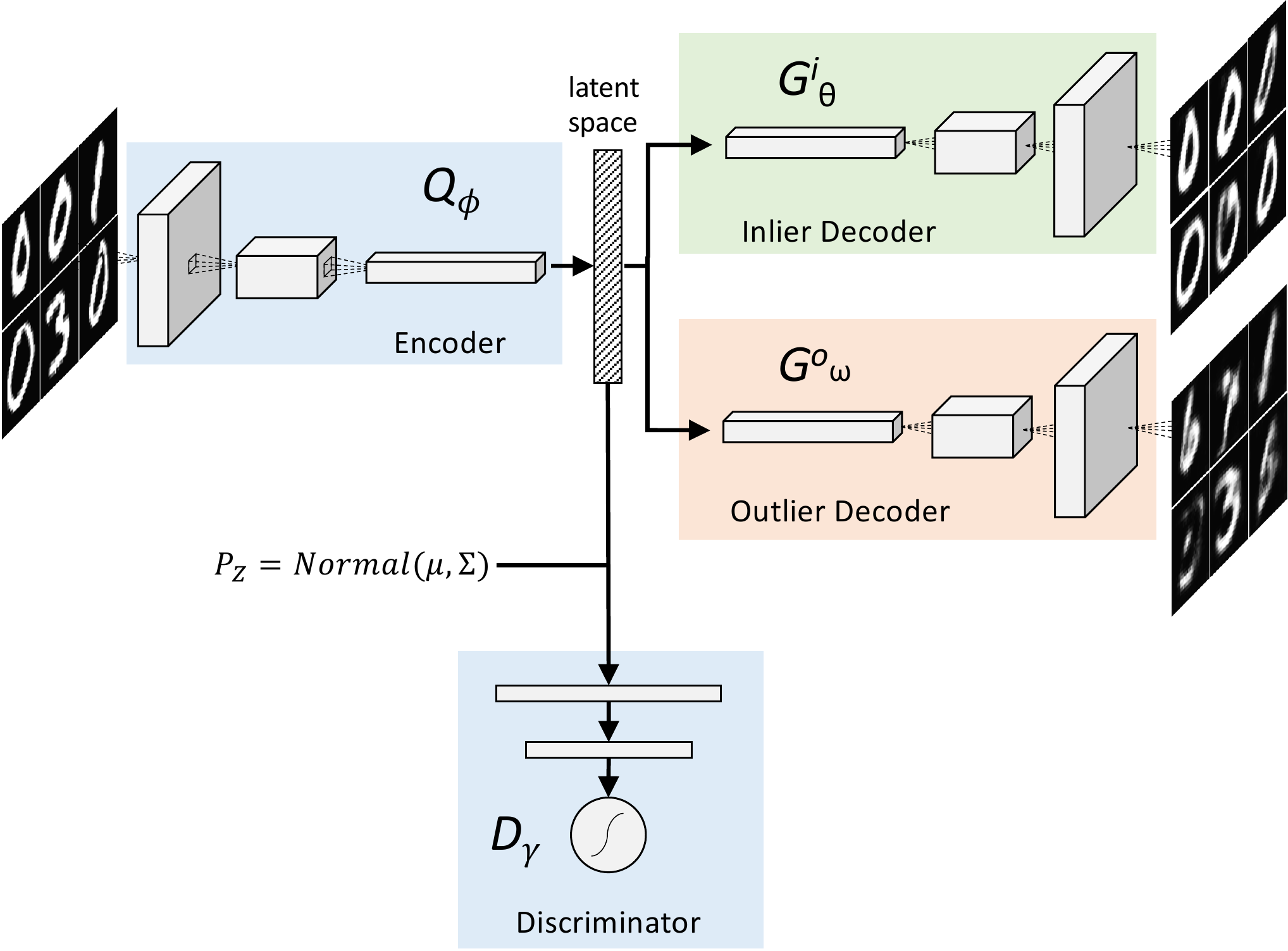}
               \caption{Architecture}
                \label{fig:architect}
        \end{subfigure}%
        \begin{subfigure}[b]{0.35\textwidth}
                \includegraphics[width=0.85\linewidth]{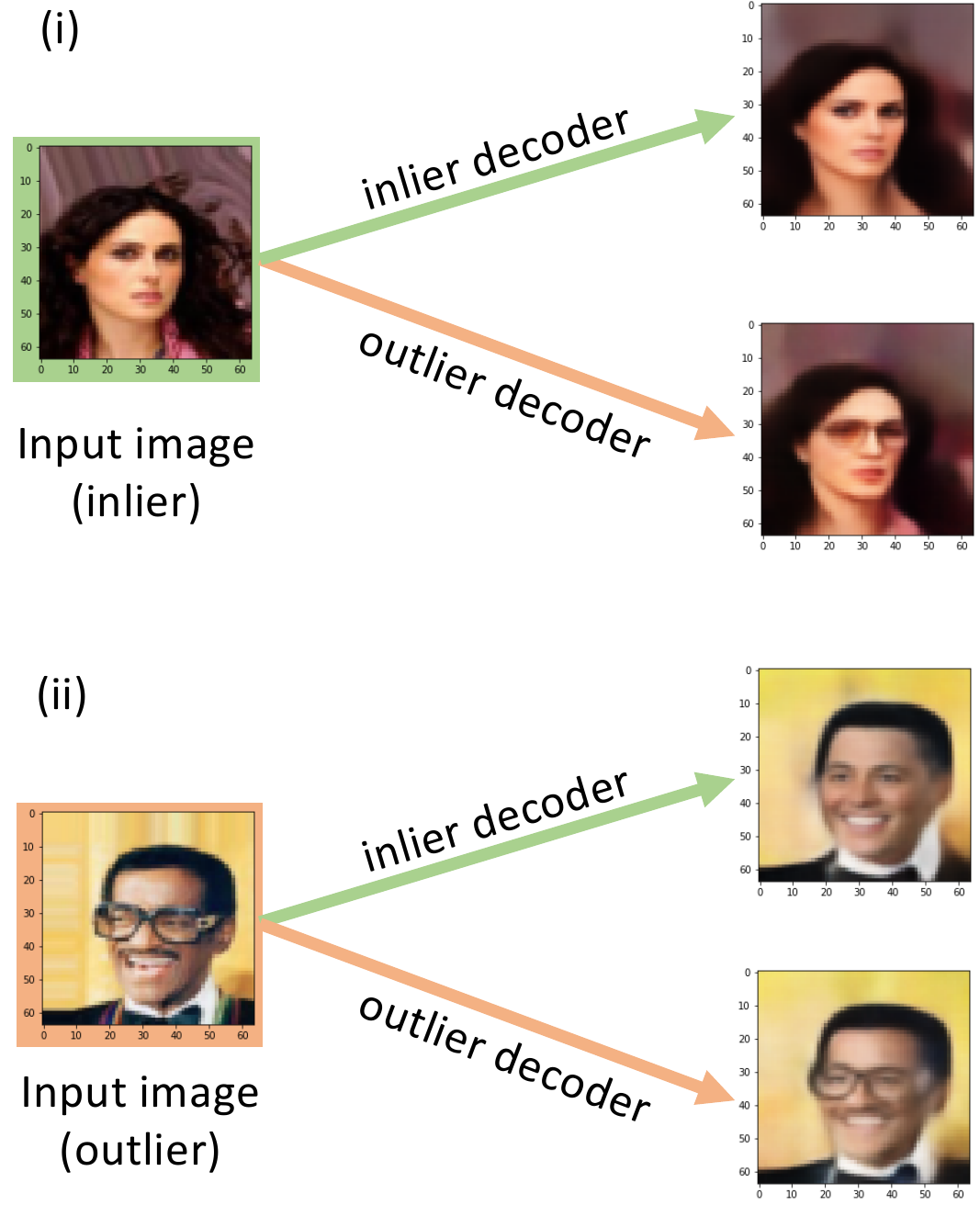}
                \caption{Application}
                \label{fig:Example}
        \end{subfigure}%
        
        \caption{WALDO Architecture and Application. (a) The architecture consists of a decoder for inliers and outliers with a common encoder trained using the Wasserstein distance; (b) A dataset of images (CelebA) where faces without glasses are inliers and with glasses are outliers: (i) An inlier input consisting of face without glasses comes out as one with glasses from the outlier decoder of WALDO. (ii) Similarly a face with glasses (outlier) comes out without glasses from the inlier decoder. }\label{fig:waldo}
\end{figure*}

To compare probability distribution we use the Wasserstein distance as an alternate to the standard
Kullback-Liebler (KL) divergence~\cite{villani}. Autoencoders based on the Wasserstein distance (WAEs) have been recently proposed and  shown to be accurate and efficient in generating complex distributions~\cite{tolstikhin2018wasserstein,patrini2018sinkhorn}. 
%For example in the task of image generation, WAEs were shown to lead to more sharper and realistic images than variational autoencoders which are based on the KL divergence and at the same are stable to train.
%For outlier detection and generation there are two properties of WD that are particularly useful: (i) the WD seamlessly {\it lifts} data from the input space into the space of probability distributions. This property is useful since we are interested in the generating mechanism for outliers, and (ii) unlike the standard KL divergence, the WD is more intuitive when the probability distributions may not have identical support. \\
To generate outlier distributions,  we leverage a recent approach to use a double decoder architecture to
distinguish between inliers and outliers~\cite{tian2019learning}. An inlier and an outlier decoder (with a common encoder) compete with each other for
data points based on reconstruction error and thus can be identified without setting a threshold parameter. \\

Our approach, \Wa,  shown in Figure 1(a), encapsulates the double
decoder framework using the Wassertein distance, resulting in a generative detection model. For detection, the predicted class of a sample is given by the decoder with the least reconstruction error. For generation, a random sample in the latent space results in a generated inlier from the inlier decoder, and in a generated outlier in the outlier decoder. 
%The inlier decoder can be used for generating inliers, while outliers are generated using the outlier decoder. 
Figure 1(b) shows how an inlier image in input (inliers are faces without glasses) is left unperturbed by the inlier decoder but the outlier decoder adds glasses to the face. Similarly an outlier image (face with glasses) is not changed by the outlier decoder but the inlier decoder removes the glasses. Similar transformations can be obtained indirectly by GANs but in a more complicated manner, by first creating a mean deviation vector from the decoder and then feeding it back through the generator~\cite{RadfordMC15}.\\

The rest of the paper is structured as follows. In Section 2 we review related work with focus on deep learning models for outlier detection. We provide a short self-contained introduction to Wasserstein distance in Section 3 as that is a key building block of our approach. In Section 4, the \Wa autoencoder architecture is introduced consisting of four distinct components, together with the algorithm to train the model. A short theoretical analysis of our approach is the focus of Section 5. The experimental set up  and results is the subject of Section 6. We conclude with a short discussion and directions for future work in Section 7.

\section{Related Work}
Outlier detection is an extensively studied topic with diverse applications~\cite{charubook,chandola2007outlier}. With the advent of deep learning, variational auto-encoders (VAEs), generative adversarial networks (GANs) and other methods have   been proposed for outlier detection~\cite{chalapathy2017robust}.
While there are many recent works on deep learning models for outlier detection (\cite{chalapathy2019deep},\cite{deviation2019}), we will primarily survey the robust and generative ones in particular, as that is the focus of the paper.

\begin{comment}
An autoencoder is a nonlinear mapping $h$ of the form:\\
\[
h: \Xa \xrightarrow{\phi} \Za \xrightarrow{\psi}\Xa
\]
\end{comment}
\noindent
{\bf Robust Methods:} A  customization of autoencoders for outlier detection is to make them robust, i.e., the model is not disproportionately effected by the presence of outliers. This in turn makes it easier to detect outliers as they will tend to have a higher reconstruction error.  For example, Zhou et. al.~\cite{zhou2017anomaly} propose an autoencoder which decomposes the input data $X$ as sum of a low-dimensional manifold $U$ plus a sparse component $S$. A robust version of VAEs was recently proposed by Akrami et. al.~\cite{akrami2019robust} by using  $\beta$-divergence instead of KL 
divergence as a loss function. \\
%The relationship between $\beta$-divergence and influence functions in statistics, which are a measure of robustness, is comprehensively overviewed in~\cite{futami2017variational}.

\noindent
{\bf Generative Adversarial Networks (GANs)} have been recently extended for anomaly detection~\cite{di2019survey}. Like variational autoencoders, GANs can map a random distribution (e.g., Gaussian)  to an arbitrary data generating distribution $P_X$, through the use of a discriminator. Thus if we have a sample of ``normal'' data points then, in principle, we can learn the normal data manifold. However, BiGANs also learn an encoder function $E$ with the property that $E = G^{-1}$, where $G$ is the generator. Now given a query point $x$, if the reconstruction error $\|x - G(E(x))\|$ is large then $x$ is likely to be an outlier. In the experiment section, we will use one representative BiGAN as a baseline against which we will compare our proposed approach~\cite{zenati2018adversarially}.\\

\begin{comment}
\begin{table}[t]
    \centering
    \begin{tabular}{|c|c|c||} \\ \hline
    {\bf Method} & {\bf Extension for OD} & {\bf Citation} \\ \hline 
      Autoencoder   & Robust Autoencoder  & Zhou \\ \hline
      Variational Autoencoder & Robust VAE & Herati \\ \hline
      GAN & BiGAN & \\ XXX \\ \hline
      WAE & ICLR & XXXXX \\ \hline
      CORA & AAAI & XXXX \\ \hline
      OURs & WAE + CORA & \\ \hline
    \end{tabular}
    \caption{Caption}
    \label{tab:my_label}
\end{table}
\end{comment}

\noindent{\bf Threshold-Free Models:}
After a model has been built, outliers can either be identified based on ranking or thresholding. For example, if reconstruction error (RE) is used as a measure of outlierness, then data points can be either ranked based on the RE score or  a threshold $\tau$ can be used such that those  points whose $RE > \tau$, are labeled as outliers. Tian et. al. recently ~\cite{tian2019learning}, proposed an autoencoder (CoRA) which uses two decoders: one for inliers and the other for outliers. Data points whose RE error is lower for the inlier decoder compared to the outlier decoder, were labeled as inliers and vice-versa. The use of two decoders frees the system from setting a pre-defined threshold value to identify outliers. We will use the idea of two decoders, in conjunction with the Wasserstein distance, to design an inlier and  an outlier generative model.\\

\noindent{\bf PU-Learning:} Positive and Unlabeled (PU) learning has a similar set up, i.e., we are
given samples from a positives and unlabeled classes along with the class prior ratio. The 
%$P^{i}_X$ and $P^{u}_X = (1-\nu)P^{i} + \nu P^{o}_X$ with class prior $\nu = Pr(X = o)$ 
classical positive and negative loss function can be expressed as a linear
combination of a modified loss over the inlier and unlabeled distributions~\cite{kiryo2017positive,li2005learning}. However, our model does not require the class prior ratio and furthermore our approach is generative and geared towards outlier detection and not classification.

\section{Wasserstein Distance}
%We provide a short but self-contained introduction to  Wasserstein distance (WD) and  %Optimal Transport (OT) theory with a particular emphasis on anomaly %detection~\citep{villani,panaretos}.

Wasserstein Distance is a measure of dissimilarity between two probability distributions just like the KL divergence. Intuitively, Wasserstein Distance measures the amount of work required to move and transform one pile of sand to another and that is why a special case of it is referred to as the Earth Movers distance~\cite{villani,panaretos}.
\begin{comment}
Wasserstein Distance enjoys  properties which make it particularly suitable for machine learning in general and anomaly detection in particular.  These include (i) it is a proper distance metric, (ii) it is meaningful and intuitive even when the two probability distributions it compares have non-overlapping (and even disjoint) support, (iii) it naturally {\it lifts} data points from the input space to the space of probability distributions.
\end{comment}

While there are several equivalent ways to define Wasserstein Distance, we will use what is sometimes called as the probabilistic definition. The \textit{p-}Wasserstein distance $W_{p}$ between a probability measure $\mu_1$ and $\mu_2$ on $\mathbb{R}^{d}$ is defined as
\[
W_{p}(\mu_1,\mu_2) = \inf_{{\substack{X \sim \mu_1}\\{Y \sim \mu_2}}}(\E\|X - Y\|^p)^{1/p}
\]

 \subsection{Wasserstein Autoencoders (WAEs)}

An autoencoder based on $W_p$ (WAE) was recently proposed~\cite{tolstikhin2018wasserstein}.
Consider an autoencoder $h: X \xrightarrow{Q} Z \xrightarrow{G} X$. Let $P_X$ be the original distribution and $h\#P_X$ be the output distribution induced by $h$.  Then a WAE learns  a function $h$ which minimizes
$W_p(P_X,h\#P_X)$. However, both the encoding ($Q$) and decoding function
($G$) can be viewed in a probabilistic fashion. Thus, if $Q(Z|X)$ is the encoding distribution and $P_G$ is the decoding distribution on $X$, then the $W_p$ between $h\#P_X$ and $P_X$ can be decoupled and  expressed in terms of $P_X$ and $P_G$.
\[
W_p(P_X,P_G) = \underset{Q:Q_Z = P_Z}{\inf}\E_{P_X}\E_{Q|Z}{\|X - G(Z)\|_p}
\]
Here $Q(Z) = \E_{X}Q(Z|X)$. To find the $P_G$ which minimizes $W_p(P_X,P_G)$,  the constraint
$Q_Z = P_Z$ is relaxed and the following objective is proposed.
\begin{align*}
D_{\text{WAE}}(P_X,P_G) = &
\underset{Q(Z|X)}{\inf}\E_{P_X}\E_{Q|Z}\|X - G(Z)\|_p \\
&+  \lambda.\mathcal{D}_Z(Q_Z,P_Z)
\end{align*}

$\mathcal{D}$ is a divergence and in the WAE-GAN version (which we will use), $\mathcal{D_Z}(Q_Z,P_Z) = D_{JS}(Q_Z,P_Z)$ is used and learned
in an adversarial manner. $D_{JS}$ is the symmetric KL divergence.

\section{Problem Definition and WALDO}
%\vspace{-0.5cm}
We now formally define the problem and propose
the \Wa architecture as a solution. \\[2ex]
\noindent
{\bf Given:}  $P^{u}_X = (1- \nu)P^{i}_X 
+ \nu P^{o}_X$ be a mixture of an inlier 
and an outlier distribution on an input space $X = \mathbb{R}^{d}$  for $ 0 < \nu < 1$. No assumptions are made on the analytical
form of the three distributions: $P^{i}_X, P^{o}_X$ and
$P^{u}_X$.  Let $X^{u}$ and $X^{p}$ be samples
from $P^{u}_X$ and $P^{i}_{X}.$ \\[2ex]
{\bf Objective:}
Learn generating distributions $P^{i}_G$ and $P^{o}_G$  on $X$ which
minimize $W_p(P^{i}_X, P^{i}_G)$ and $W_p(P^{o}_X,P^{o}_G)$.\\[2ex]
{\bf Constraints:}  We do not
have access to samples from $P^{o}_{X}.$
\subsection{WALDO Architecture} The architecture of 
\Wa is a generalization of the WAE~\cite{tolstikhin2018wasserstein} and CoRa~\cite{tian2019learning}  to simultaneously detect and generate inliers and outliers. \Wa consists of four components as shown in Figure \ref{fig:waldo}: 
\begin{enumerate}
\item An inlier decoder and generator denoted as $G^{i}_{\theta}$ which
maps the latent space $Z$ to the output space $X$.  The inlier
decoder will induce a distribution $P^{i}_G$. Once trained,
$G^{i}_{\theta}$ can take an element generated from $P_Z$ and
produce samples which will appear to be from $P^{i}_X$.
\item An outlier decoder and generator denoted as
$G^{o}_{\omega}$. Like the inlier decoder, the outlier
decoder can be used to generate outlier samples which
will appear to be from $P^{o}_X$.
\item A common encoder $Q_{\phi}$ which maps the input space
$X$ into the latent space $Z$.  In the original WAE paper,
the constraint $\E_{X \sim P_X}(Q|Z) = P_Z$ is enforced using
an adversarial discriminator loss. However in \Wa we 
have the option of either enforcing  $\E_{X \sim P^{i}_X}(Q|Z) = P_Z$ or $\E_{X \sim P^{u}_X}(Q|Z) = P_Z$. In our experiments
we have consistently observed that the former gave better
results than the latter. This is not unexpected as by only
enforcing the constraints on the inliers there will be a smaller
chance that the inliers and the outliers will be mapped to
the same region of the latent space.
\item A discriminator $D_{\gamma}$,trained
in an adversarial manner like in traditional GANs. The role of $D_{\gamma}$ is to enforce the
constraint $Q_Z = P_Z$. However, unlike traditional GANs, $D_{\gamma}$ operates
in the lower-dimensional latent space $Z$. Recall in adversarial learning,
the encoder $Q$ is trying to ``fool'' the discriminator to treat its samples
as those from the prior $P_Z$. 
\end{enumerate}

\begin{algorithm*}[h]

\SetAlgoLined
\SetKwInOut{Input}{Input}

\Input{positive data $X^{p}$, unlabeled test data $X^{u}$}

 Initialize the parameters of the encoder $Q_{\phi}$, inlier decoder $G^{i}_{\theta}$, discriminator $D_{\gamma}$. Set the outlier decoder $G^{o}_{\omega}$ parameter $\omega=\theta$.
 
 \While{$\phi,\theta,\omega$ not converged}{
 
  Sample  batch of size $n$  from %$X = X^{p} \cup X^{u} $ with 
  $X^{p}$ positive  and $X^{u}$ unlabeled data, $|X^{p} \cup X^{u}| = n$

  Sample $\{z_{1},\dots,z_{n}\}$ from the prior $P_{Z}$\\
  Sample $\{\hat{z}_{1},\dots,\hat{z}_{n}\}$ from $Q_{\phi}(x_{i})$ for $i = 1,\dots,n$
 
  \BlankLine

  Update $D_{\gamma}$ by ascending:
  \[
 \frac{\lambda}{n} \sum_{j=1}^{n} \log D_{\gamma}(z_{j}) + \log(1- D_{\gamma}(\hat{z}_{j}))
  \]
Compute advantage of inlier decoder:
 \[
 adv^{i} = \min_{\forall x_i \in X^{u}} \Vert G^{o}_{\theta}(\hat{z}_{j}) -  x_{j} \Vert^{2}_{2}  -  \min_{\forall x_j \in X^{p}} \Vert G^{i}_{\theta}(\hat{z}_{j}) -  x_{j} \Vert^{2}_{2}
 \]

\For{$j = 1,\dots,n$}{%
    \eIf{$\Vert G^{i}_{\theta}(\hat{z}_{j}) -  x_{j} \Vert^{2}_{2} + adv^{i} < \Vert G^{o}_{\omega}(\hat{z}_{j}) -  x_{j} \Vert^{2}_{2} \vee x_{j} \in X^{p}$}{
    $y_{j} = 0$
    }{
    $y_{j} = 1$
}
   }

Update $Q_{\phi}$, $G^{i}_{\theta}$ and $G^{o}_{\omega}$ by descending:
\[
\frac{1}{n} \sum_{j=1}^{n} ( y_{j} \Vert G^{o}_{\omega}(\hat{z}_{j}) -  x_{i} \Vert^{2}_{2} + (1-y) \Vert G^{i}_{\theta}(\hat{z}_{i}) -  x_{j} \Vert^{2}_{2}) 
- \lambda \cdot( \log(D_{\gamma}(\hat{z}_{j}) ) )
\]

 }
 \caption{Wasserstein Autoencoder for Learning Distribution of Outliers (WALDO)}
 \label{alg:waldo}
\end{algorithm*}

%\begin{wrapfigure}{r}{0.5\linewidth}
% \begin{minipage}[H]{1\textwidth}
% \scalebox{1}{\myalgorithm}
% \end{minipage}%
%\end{wrapfigure}

\subsection{Algorithm}
WALDO is defined in Algorithm~\ref{alg:waldo}. First the discriminator is trained by ascending (line~7) to discriminate between samples from the prior $P_z$ and samples from the encoder $Q_{\theta}$. In practice only encoded inliers samples will be forced to match the prior distribution ({\bf positive-only} $D_\gamma$ training).
In the training of the autoencoder, only the decoder with lower reconstruction error will be selected in the loss for each data point (lines~9-16). Note that in the competition for a data point, the inlier decoder has seen more samples, thus it has a natural {\bf advantage} over the outlier decoder in decoding both outliers and inliers. Conversely but less frequently, a random initialization may lead to an advantage of the outlier decoder, with a consequently spurious training during the initial epochs.
To cope with both cases of imbalance we introduce an {\bf advantage} term (line 8), that penalizes the reconstruction error of the decoder with the best reconstruction error.

\section{Analysis of WALDO}
We analyze theoretical aspects of 
\Wa for the special case of $p=2$. In particular, we show that under certain circumstances,  $W_2(P^{u}_X,P^{o}_G)$ upper bounds a positive weighted sum of $W_2(P^{o}_X,P^{o}_G)$, $W_2(P^{i}_X,P^{i}_G)$ and $W_2(Q^Z,P^Z)$.  Thus by minimizing
an upper bound we can indirectly optimize the decoders. We use the following 
characterization of WAE~\cite{patrini2018sinkhorn} for decoders with the added assumption that they are  Lipschitz with constant $\gamma$.
%\vspace{-0.1 cm}
\begin{align*}
W_{2}(P_X,P_G) 
= &\inf_{Q}\sqrt{\E_{X \sim P_X}\|X - G(Q(X))\|^{2}} \\
&+ \gamma.W_{2}(Q_Z,P_Z)
\end{align*}
\begin{thm}
For a system with an inlier decoder $P^{i}_G$ and an outlier decoder $P^{o}_G$ and a shared deterministic encoder $Q$, the following holds:
\begin{align*}
%W_2(P^{u}_X,P^{o}_G) \geq \alpha W_{2}(P^{o}_X,P^{o}_G) +
%\beta W_{2}(P^{i}_X,P^{i}_G) + \delta W_{2}(Q_z,P_z)
W_2(P^{u}_X,P^{o}_G) \geq \sqrt{\frac{\nu}{2}}W_{2}(P^{o}_X,P^{o}_G) +
\sqrt{\frac{1 -\nu}{2}}W_{2}(P^{i}_X,P^{i}_G) +\\ \gamma.\left(1 -\sqrt{\frac{\nu}{2}} -
      \sqrt{\frac{1-\nu}{2}}\right)W_{2}(Q_Z,P_Z)
\end{align*}
\end{thm}
\begin{proof} See Supplementary Text.
\end{proof}
\noindent
{\bf Implication of Theorem:} The above theorem shows that by using the Wasserstein metric we can formally
distribute the error between the unlabeled data distribution $P^{u}_X$ and the outlier (inlier) generator $P^{o}_G (P^{i}_G)$ across the two two decoders. The coefficients $\sqrt{\frac{\nu}{2}}$ also suggests that 
if $\nu$ is very small then an algorithm which tries to minimize $W_2(P^{u}_X,P^{o}_G)$ will effectively
expend ``more effort'' in optimizing $W_{2}(P^{i}_X,P^{i}_G)$ than $W_{2}(P^{o}_X,P^{o}_G)$. The use of
Advantage in the algorithm is way to compensate the natural weakness of optimizing the outlier decoder
due to the small value of $\nu$ even though the dependence is improved as the factor is $\sqrt{\frac{\nu}{2}}$ will be higher than $\nu < 1/2.$
Note that in practice the Lipschitz condition can be enforced using gradient clipping.

\section{Experiments}
In this section we empirically evaluate the effectiveness of WALDO. We report on four sets of experiments.
\begin{enumerate}
    \item We carry out an ablation study of \Wa  by varying its internal components. Specifically, we evaluate the accuracy of \Wa  when the discriminator is applied to only inlier data, i.e., data sampled from $P^{i}_X$. Similarly the impact of training \Wa  with and without the use of {\bf advantage} is tested.
    \item We evaluate \Wa on its ability to generate outliers. We test whether $P^{o}_G$ can be used to generate
    new network intrusion attacks using the KDD99 data set.
    \item We present a real case study where \Wa is applied on real sales data to accurately discover extremely rare patterns with high recall.
    \item Finally we compare \Wa  with other state of the art and representative deep learning based methods for anomaly detection: DeepSVDD ~\cite{ruff2018deep}, ALAD ~\cite{zenati2018efficient}, WAE~\cite{tolstikhin2018wasserstein} and  CoRA ~\cite{tian2019learning}. The comparison is carried out by varying contamination level of the training set and outlier ratios in test dataset.
\end{enumerate}

\iffalse
First we carry out an ablation study to measure the improvement brought by internal features of $WALDO$. In particular we evaluate the proposed \textit{advantage penalty} and \textit{positive-only training}. Then, we provide a thorough comparison with other anomaly detection methods: DeepSVDD ~\cite{ruff2018deep}, ALAD ~\cite{zenati2018efficient}, CoRA ~\cite{tian2019learning}. We evaluate each model by considering: (a) training dataset with no contamination, 5\% contamination, 10\% contamination (b) test dataset with several outlier ratios $\nu:\{0.05,0.1,0.2,0.5\}$. We set different contamination ratio for these two cases. (iii) we visualize the effect of the $WALDO$ on outliers by comparing inliners, outliers and their reconstructions. we show the capability of generating new outliers on two datasets $(KDD, retailer)$
\fi

\subsection{Datasets}
We use four publicly available datasets for experiments:
\begin{enumerate}
\item \textbf{MNIST} \cite{mnist}: containing $60$k training samples and $10$k test samples from 10 digit classes. Each digit is a $28\times 28 $ grayscale image. We choose the digit $0$ as the inlier class and the others as outliers.
\item \textbf{Fashion MNIST} \cite{fmnist}: consisting of $60$k training samples and $10k$ test samples from $10$ classes. Each sample is a $28\times 28 $ grayscale image in a clothes category. We use the class $0$ as inliers ($X^{i}$), and the others as outliers.
\item \textbf{KDD99} \cite{kdd}: a large-scale network traffic data with 121 features in each sample. We use $10\%$ of the dataset to extract the inliers ($X^{i}$)  and another $10\%$ for the unlabeled data. This data set is also used to show the capability of \Wa to generate new meaningful attacks.
\item \textbf{CIFAR10} \cite{cifar10}: consisting of $60$k $32\times32$ color images in 10 classes including $50$k training and $10$k test images. 
%and a private dataset from a large retailer \footnote{Name hidden for double blind.}.
\end{enumerate}

\subsection{Ablation study. }

\begin{figure*}[h!]
\begin{subfigure}{.45\textwidth}
\centering
  \includegraphics[width=1\linewidth]{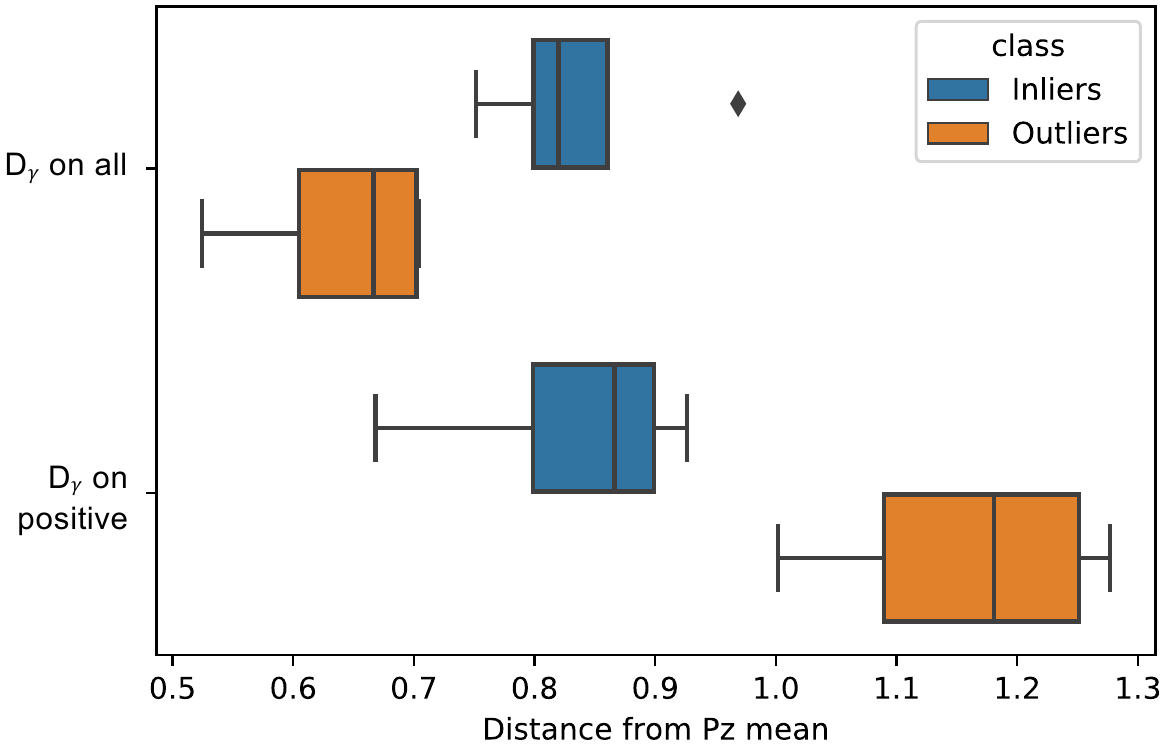}
  \caption{Distance of outliers from $P_Z$ mean increases.}
  \label{fig:pu}
  \end{subfigure}
  \quad\quad\quad
  \begin{subfigure}{.45\textwidth}
  \centering
  \includegraphics[width=1\linewidth]{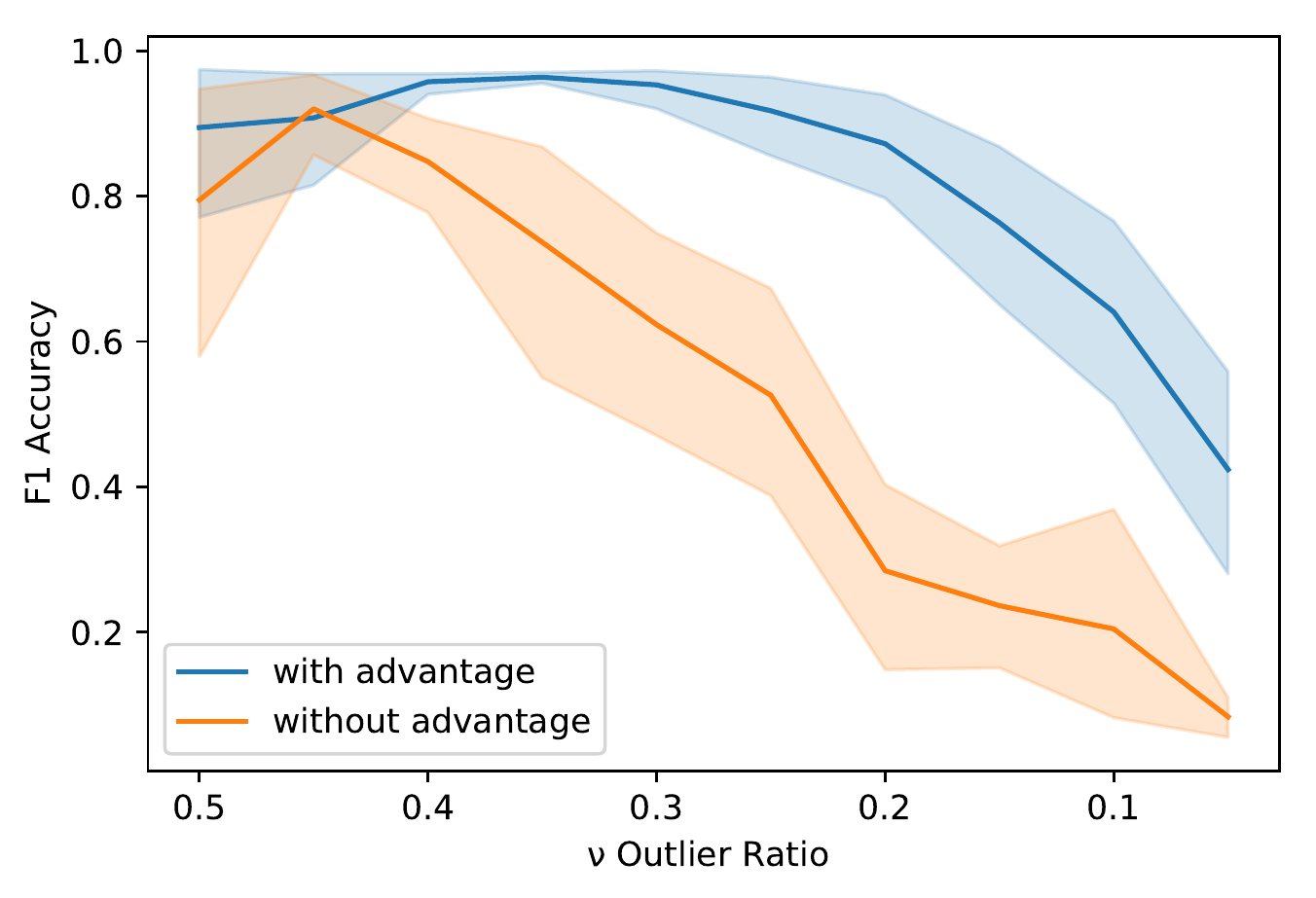}
  \caption{Impact of Advantage}
  %\caption{Distribution of residuals become more separated. Residuals are computed as $\Vert G^{I}_{\theta}(\hat{z}_{i}) -  x_{i} %\Vert^{2}_{2} - \Vert G^{O}_{\omega}(\hat{z}_{i}) -  x_{i} \Vert^{2}_{2} $. }
  \label{fig:advantage}
  \end{subfigure}
  \caption{Impact of positive training  on $D_{\gamma}$ and using the Advantage for Generators}
  \label{fig:positive}
\end{figure*}

\textbf{Impact of Positive-only $D_{\gamma}$ on \Wa}.
%\textbf{Positive-only discriminator training} 
Training the discriminator $D_\gamma$ only on the positive data (labeled inliers) helps the model in separating the two distributions in the latent space, having effects both on the latent space and on the output space. In Figure \ref{fig:pu} we aggregate the effects of  positive training for $D_\gamma$, showing the distance of the encoded samples from the mean of the distribution $P_Z$.
As expected, the outliers get mapped further from the mean of the prior distribution $P_Z$, while the inliers are closer.

\iffalse
By mapping inliers and outliers far from each other in the latent space, the difference in reconstruction error also increases: reconstructing the outliers becomes harder for the inlier decoder, and the same is true for the inliers reconstructed by the outliers decoder. Notice that the mean distance of outliers increases when trained with only positive data. This difference is better show in Figure \ref{fig:residuals}, where the distribution of residuals is plotted with and without the positive-only training. Residuals are computed as $\Vert G^{i}_{\theta}(\hat{z}_{j}) -  x_{j} \Vert^{2}_{2} - \Vert G^{o}_{\omega}(\hat{z}_{j}) -  x_{j} \Vert^{2}_{2} $, therefore samples with negative residuals would be classified as $y=0$ in Algorithm \ref{alg:waldo}, samples with positive residuals as $y=1$. Without positive-only training, the two distributions are closer, leading to a higher number of misclassified samples.
\fi

\iffalse
\begin{figure}[hbt]
  \includegraphics[width=1\columnwidth]{pics/figure_advantage.pdf}
  \caption{Effect of positive training the discriminator $D_\gamma$ on the residuals. Residuals are computed as $\Vert G^{I}_{\theta}(\hat{z}_{i}) -  x_{i} \Vert^{2}_{2} - \Vert G^{O}_{\omega}(\hat{z}_{i}) -  x_{i} \Vert^{2}_{2} $. }
  \label{fig:residuals}
\end{figure}
\fi

%\subsection{Evaluation Metrics}
%We use two well-known evaluation measures in anomaly detection:  Area Under Recevier Operating Characteristic Curve (AUC for short) and Area Under Precision-Recall Curve (AUPRC).

\textbf{Impact of Advantage on \Wa}.
Recall that in \Wa the two decoders, $G^{i}$ and $G^{o}$
compete with each other to get points assigned to them.
However because of the availability of the $X^{i}$ set,
$G^{i}$ has a natural advantage to have a low reconstruction
error on data points in $X^{u}$. To overcome the natural
bias, we introduced the \textbf{Advantage} term (see Line 9 in Algorithm 1).
%\textbf{Advantage penalty} The role of the advantage penalty is to avoid that one of the two decoders has a lower reconstruction error on the wrong data points, for the sole reason that this decoder has seen many more samples. 
In Figure \ref{fig:advantage} we show how using the advantage penalty substantially improves the results (higher F1 score), in particular when the outlier ratio becomes smaller. 
%The use of \textbf{Advantage} also shows why the results in the CoRA paper~\cite{tian2019learning} tend to be more accurate for higher outlier ratios.
Moreover, we observed how employing the advantage penalty resulted in more reproducible results over different random seeds, as can be seen from the smaller variance in the accuracy.

\subsection{Extreme Outlier Discovery: Real Case Study}

We give an example of how \Wa can be used to detect extremely rare patterns on a real data set acquired from a large retailer. We look at the weekly sales pattern of one product $X$ over nearly four years, (from $2016/03/20$ to $2019/12/29$ in all $52$ price-markets in the US, pre COVID-19),  which is typically sold more on weekends than weekdays. Each data point is a vector of seven dimensions, and we took a small fraction of the data points and labeled them as inliers if the volume of the product sold on either Saturday or Sunday was greater than  any of the weekdays. There was a total of $10,234$ inliers and $62$ outliers. Thus the percentage of outliers was $0.61\%$. \Wa was only given a small labeled sample of inliers $(2,047)$ and all the remaining set was unlabeled. Note we chose this pattern, which is ``easy to query'' as it makes it straightforward to characterize outliers. Recall \Wa does not see any labeled outliers.
Here are the key observations:
\begin{enumerate}
\item All the outliers were assigned to the outlier decoder (high recall).
\item However, many inliers were also assigned to the outlier decoder just because of the extreme skewness of the data set (low precision).
\item If we ranked all the data points by reconstruction error of the inlier decoder we observed an average precision (AP) of close to $46\%$ for the outliers. Thus even though outliers constitute an extremely small percentage of the data, we are able to locate them at the top of the list. This demonstrates that \Wa has the promise to detect rare patterns.
\end{enumerate}

\begin{figure}[H]
    \includegraphics[width=1.\textwidth]{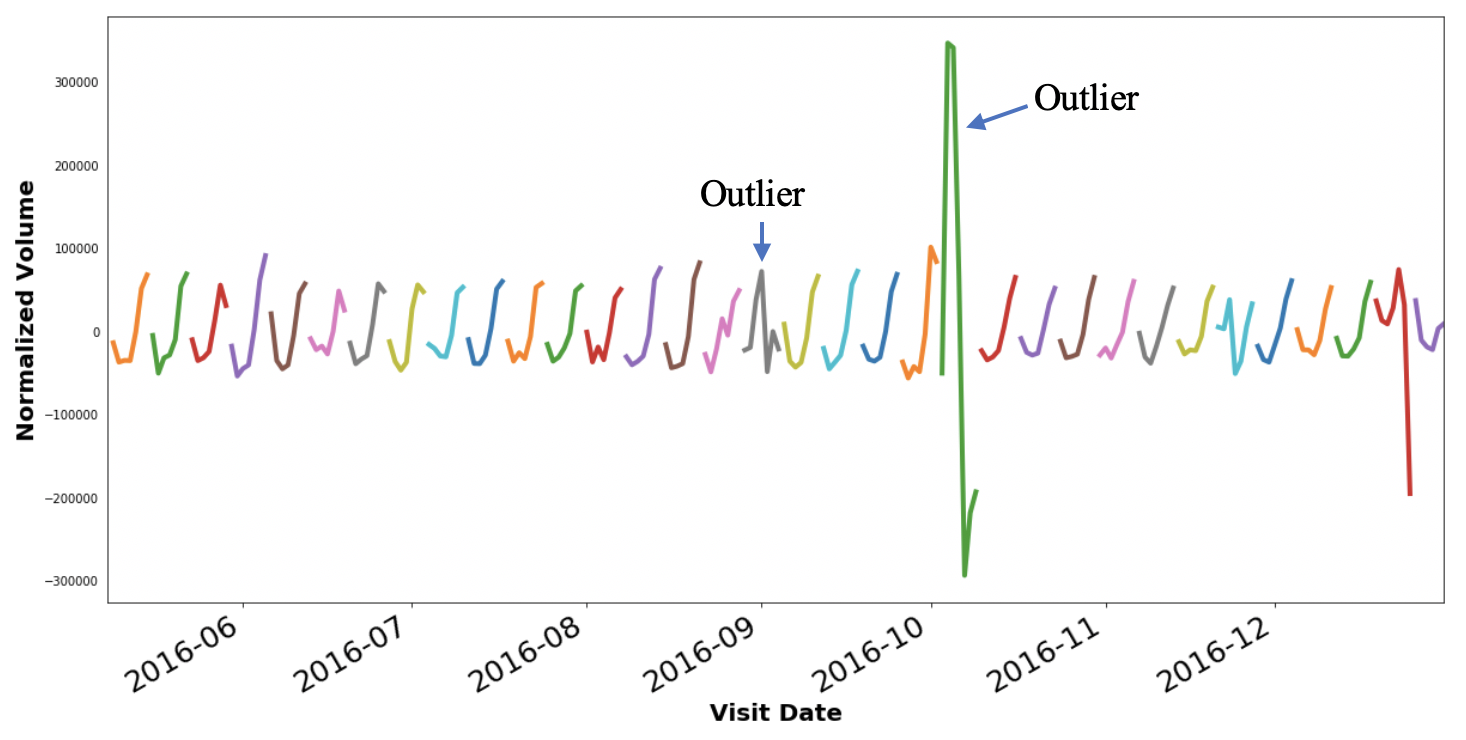}
  \caption{Outlier detection of retail patterns. Outliers constitute a very small percentage of the data.}
  \label{fig:walmart}
\end{figure}

\subsection{Capability of generating new attacks on the KDD99 Data Set}

One of the novelties of our work is to provide the capability of generating new outliers. For example we can generate realistic new attacks using the KDD99 dataset. Our attacks are generated using the trained \Wa network architecture for the data set. Both positive and unlabeled were used in the training model. We used trained encoder on the inliers and outliers independently from KDD99 to generate the encoded data in which two distributions are formed. From the two distributions, we sampled two independent groups of Gaussian random variables $X_{i}  \sim \mathcal{N}(P_{Z_i},\,\sigma_{i}^{2})\,$, $X_{o}  \sim \mathcal{N}(P_{Z_o},\,\sigma_{o}^{2})\,$. Then we decoded the samples to reconstruct the data.

\begin{figure}[H]
\centering
  \includegraphics[width=0.9\textwidth]{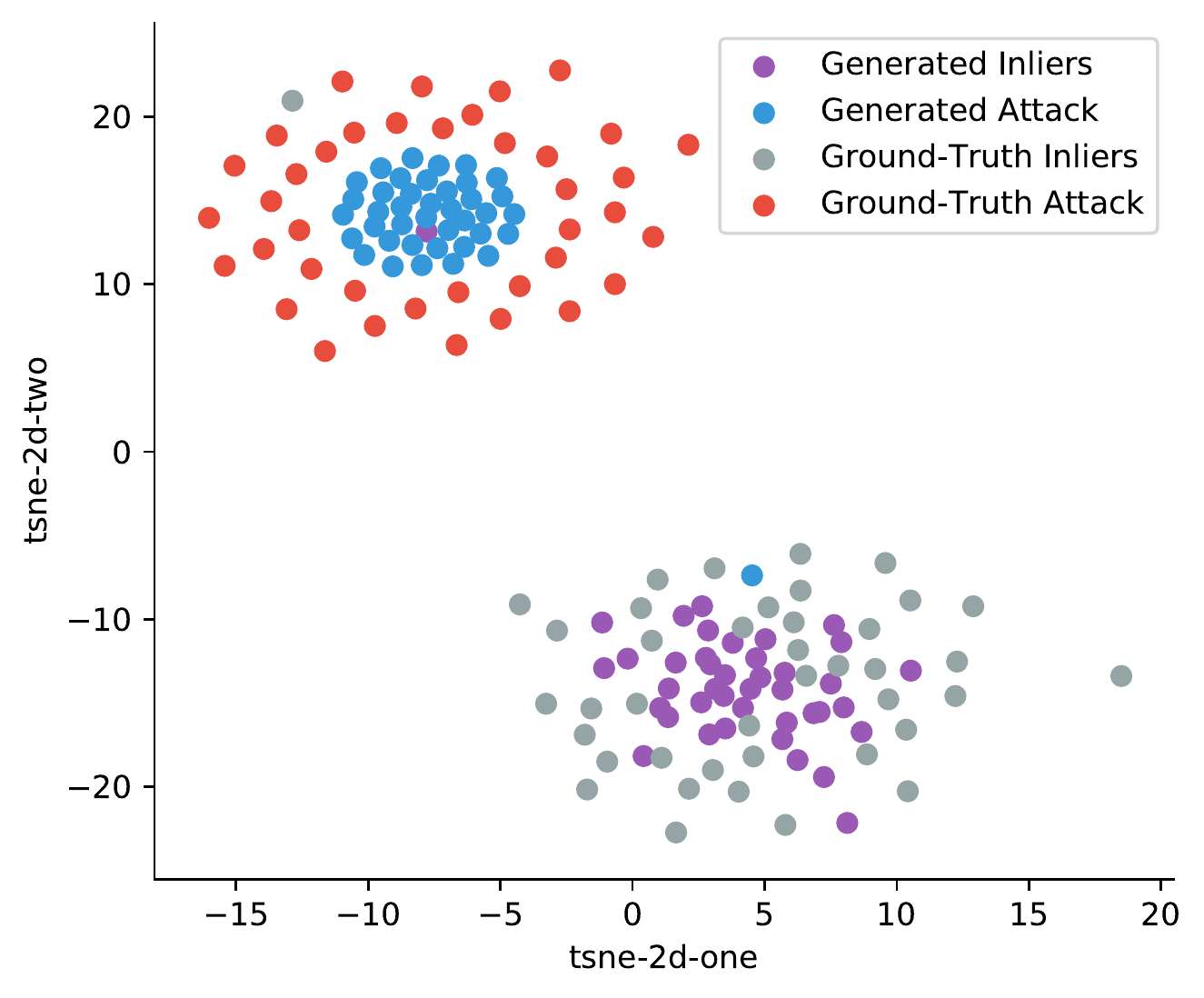}
  \caption{KDD99 Outlier Generation }
  \label{fig:gen_kdd}
\end{figure}

 \iffalse
 \begin{minipage}[H]{0.45\textwidth}
\includegraphics[width=1\textwidth]{pics/kddoutput.pdf}

    \captionof{figure}{Generated Inliers and outliers from KDD99 Dataset by randomly sampling data points from prior distribution $z$.}%
    \label{fig:gen_kdd}%
\end{minipage}%
\fi

Figure \ref{fig:gen_kdd} shows the real and generated network traffic samples in a t-distributed stochastic neighbor embedding(TSNE) plot. In order to quantitatively assess the quality of the
generated data, we used the Euclidean Distance to evaluate on $10^3$ samples. It can be seen in the TSNE plot that the distribution of Generated attacks is tighter than the distribution of inliers, which means that Wasserstein distance $W(P^{o}_X,P^{o}_G)$ is larger than $W(P^{i}_X,P^{i}_G)$. One explanation to this interesting observation is that we do not direct optimize the distance between $P^{o}_X$ and $P^{o}_G$. Thus \Wa can effectively create attacks and normal network traffic data which can be used to detect anomalies in any network examination system.

\subsection{Convergence Efficiency} 
\begin{figure}[h!]
    \includegraphics[width=0.9\textwidth]{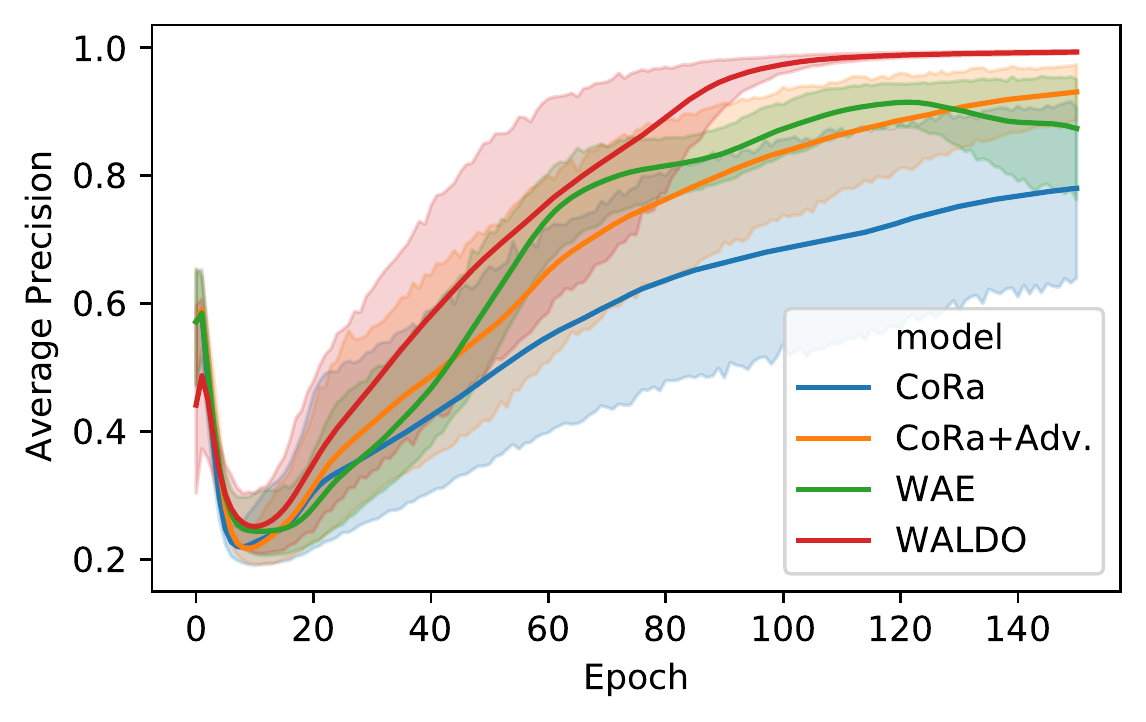}
    \centering
  \caption{Average Precision on testing set for WALDO, WAE~\cite{tolstikhin2018wasserstein} and CoRa~\cite{tian2019learning}. }
	\label{fig:convergence}
\end{figure}
In this experiment, we measure the performance of \Wa during  training by tracking the number of epochs needed to converge to a stable result on MNIST. 
In Figure~\ref{fig:convergence}, the results of \Wa is compared to CoRa~\cite{tian2019learning}, both in its original setting and improved with the proposed Advantange penalty, and the Wasserstein Autoencoder (WAE ~\cite{tolstikhin2018wasserstein}, showing the average precision (AP) mean and standard deviation over eight runs on the unlabeled set.
The results show that \Wa reaches higher precision faster, while convergence is obtained consistently with very small variance across different runs.

\begin{table*}[h!]

\centering
\resizebox{\linewidth}{!}{

\begin{tabular}{lcrr|rr|rr|rr|rr}
\cmidrule{3-12}

& & \multicolumn{2}{c|}{\textbf{DeepSVDD \cite{ruff2018deep}} }
& \multicolumn{2}{c|}{\textbf{ALAD \cite{zenati2018adversarially}}} 
& \multicolumn{2}{c|}{\textbf{WAE \cite{tolstikhin2018wasserstein}}} 
& \multicolumn{2}{c|}{\textbf{CoRa \cite{tian2019learning} + Advantage}} 
& \multicolumn{2}{c}{\textbf{WALDO}}  \\ 

\cmidrule{3-12}

&  $\nu$ & \multicolumn{1}{c}{AUC} & \multicolumn{1}{c|}{AUPRC}  
 &  \multicolumn{1}{c}{AUC}   & \multicolumn{1}{c|}{AUPRC} 
 &  \multicolumn{1}{c}{AUC} & \multicolumn{1}{c|}{AUPRC} 
 &  \multicolumn{1}{c}{AUC}  & \multicolumn{1}{c|}{AUPRC}
 & \multicolumn{1}{c}{AUC} & \multicolumn{1}{c}{AUPRC}\\ 
\toprule

\multirow{ 4}{*}{\textbf{MNIST}} & 0.05
& $\boldsymbol{99.50  \pm 0.01} $
& $\boldsymbol{92.92  \pm 0.0}$
& $93.42  \pm 2.73$ 
& $37.92  \pm 14.82$ 
& $97.54  \pm 1.41$ 
& $61.55  \pm 20.91$
& $98.93  \pm 0.62$ 
& $77.55  \pm 16.34$ 
& $99.33  \pm 0.56$ 
& $90.30  \pm 6.58$
\\
 & 0.1
& $99.26  \pm 0.15 $
& $93.81  \pm 1.3$
& $94.15  \pm 1.68$ 
& $58.73  \pm 9.68$ 
& $98.24  \pm 0.92$ 
& $79.75  \pm 13.56$
& $99.13  \pm 0.73$ 
& $89.19  \pm 12.39$ 
& $\boldsymbol{99.77  \pm 0.13}$ 
& $\boldsymbol{98.07  \pm 0.99}$
\\
 & 0.2
& $98.96  \pm 0.29$ 
& $96.11  \pm 6.9$
& $93.62  \pm 3.29$ 
& $73.93  \pm 14.24$ 
& $95.97  \pm 4.08$ 
& $86.35  \pm 13.64$
& $98.90  \pm 0.29$ 
& $95.91  \pm 2.33$ 
& $\boldsymbol{99.82  \pm 0.07}$ 
& $\boldsymbol{99.29  \pm 0.28}$
\\
 & 0.5
& $98.46  \pm 0.13 $
& $98.56  \pm 0.1$
& $92.45  \pm 1.73$ 
& $89.70  \pm 3.13$ 
& $88.93  \pm 5.89$ 
& $88.67  \pm 4.67$
& $96.71  \pm 1.36$ 
& $96.85  \pm 0.74$ 
& $\boldsymbol{99.13  \pm 0.28}$ 
& $\boldsymbol{99.14  \pm 0.28}$
\\ \midrule

\multirow{ 4}{*}{\textbf{FMNIST}} & 0.05
& $\boldsymbol{90.39  \pm 0.03}$
& $\boldsymbol{41.67  \pm 3.75}$
& $70.65  \pm 2.03$ 
& $32.85  \pm 3.11$ 
& $89.77 \pm 0.41$ 
& $38.63  \pm 2.40$
& $86.84  \pm 0.91$ 
& $43.45  \pm 2.92$ 
& $89.48  \pm 0.66$ 
& $41.65  \pm 1.41$
\\
 & 0.1
& $89.67  \pm 1.17 $
& $58.72  \pm 2.82$
& $70.80  \pm 1.03$ 
& $47.38  \pm 2.34$ 
& $\boldsymbol{91.23  \pm 0.56}$ 
& $58.88  \pm 0.88$
& $86.82  \pm 0.97$ 
& $59.35  \pm 3.69$ 
& $90.01  \pm 0.35$ 
& $\boldsymbol{64.90  \pm 1.71}$
\\
 & 0.2
& $88.46 \pm 0.54 $
& $68.31  \pm 1.86$
& $71.26 \pm 1.08$ 
& $59.81  \pm 4.13$ 
& $\boldsymbol{91.17 \pm 0.76}$ 
& $72.56  \pm 2.07$
& $84.20  \pm 1.38$ 
& $66.63  \pm 3.50$ 
& $88.47  \pm 1.18$ 
& $\boldsymbol{72.87  \pm 2.96}$
\\
 & 0.5
& $86.99  \pm 0.20 $
& $85.68 \pm 0.1$
& $70.20  \pm 2.21$ 
& $64.94  \pm 3.02$ 
& $\boldsymbol{90.77  \pm 0.44}$ 
& $\boldsymbol{89.81  \pm 0.38}$
& $82.91  \pm 1.62$ 
& $85.22  \pm 1.87$ 
& $86.73  \pm 2.06$ 
& $88.06  \pm 1.99$
\\\midrule

\multirow{ 4}{*}{\textbf{CIFAR10}} & 0.05
& $60.74  \pm 4.41 $
& $6.73  \pm 1.0$
& $76.46  \pm 1.12$ 
& $11.77  \pm 0.63$ 
& $78.88  \pm 0.02$ 
& $13.15 \pm 0.01$
& $79.05  \pm 0.09$ 
& $13.19  \pm 0.04$ 
& $\boldsymbol{79.13  \pm 0.15}$ 
& $\boldsymbol{13.19  \pm 0.02}$
\\
 & 0.1
& $58.09  \pm 4.24 $
& $11.78 \pm 1.4$
& $71.05  \pm 1.90$ 
& $20.72  \pm 1.51$ 
& $74.03  \pm 0.05$ 
& $23.25  \pm 0.03$
& $74.04  \pm 0.03$ 
& $23.22  \pm 0.02$ 
& $\boldsymbol{74.07  \pm 0.04}$ 
& $\boldsymbol{23.28  \pm 0.04}$
\\
 & 0.2
& $56.57  \pm 4.91 $
& $21.75  \pm 2.13$
& $67.52  \pm 1.34$ 
& $32.41  \pm 1.43$ 
& $70.54  \pm 0.03$ 
& $\boldsymbol{35.37 \pm 0.01}$
& $70.57  \pm 0.09$ 
& $35.32  \pm 0.06$ 
& $\boldsymbol{70.58  \pm 0.01}$ 
& $\boldsymbol{35.37  \pm 0.01}$
\\
 & 0.5
& $\boldsymbol{66.71  \pm 0.71}$
& $52.39  \pm 3.12$
& $63.53  \pm 0.83$ 
& $64.38  \pm 2.14$ 
& $66.65  \pm 0.09$ 
& $67.44  \pm 0.02$
& $58.56  \pm 8.99$ 
& $67.46  \pm 0.02$ 
& $\boldsymbol{71.64  \pm 0.24}$ 
& $\boldsymbol{67.57  \pm 0.21}$\\\midrule

\multirow{ 4}{*}{\textbf{KDD}} & 0.05
& $99.38  \pm 0.24$ 
& $84.21  \pm 5.54$
& $97.94  \pm 6.05$ 
& $74.98  \pm 3.13$ 
& $99.25 \pm 0.08$ 
& $74.21  \pm 2.50$
& $98.59  \pm 0.63$ 
& $74.76  \pm 11.48$
& $\boldsymbol{99.51 \pm 0.02}$ 
& $\boldsymbol{86.34  \pm 1.03}$ 
\\
 & 0.1
& $99.65  \pm 0.25$ 
& $91.17  \pm 3.48$
& $96.17  \pm 2.13$ 
& $83.65  \pm 4.64$ 
& $99.24  \pm 0.12$ 
& $85.84  \pm 2.70$
& $98.72  \pm 0.59$ 
& $86.85  \pm 5.39$
& $\boldsymbol{99.65  \pm 0.07}$ 
& $\boldsymbol{95.67 \pm 1.35}$ 
\\
 & 0.2
& $99.30  \pm 0.14$ 
& $95.78  \pm 0.87$
& $98.70  \pm 0.31$ 
& $92.22  \pm 0.72$ 
& $99.23  \pm 0.11$ 
& $93.19  \pm 0.39$
& $98.56  \pm 0.63$ 
& $92.57  \pm 3.12$ 
& $\boldsymbol{99.51  \pm 0.18}$ 
& $\boldsymbol{97.06  \pm 1.34}$
\\
 & 0.5
& $\boldsymbol{99.79 \pm0.22}$ 
& $99.20  \pm 0.53$
& $98.65 \pm 0.24$ 
& $93.25  \pm 0.42$ 
& $99.39  \pm 0.18$ 
& $98.41  \pm 0.62$
& $98.60  \pm 0.71$ 
& $97.91  \pm 1.04$ 
& $99.61  \pm 0.18$ 
& $\boldsymbol{99.42  \pm 0.18}$\\

\bottomrule
\end{tabular}}
\caption{Evaluation and comparison with state of the art models, showing average values of multiple runs for AUC and AUPRC. In \textbf{bold}: best result for selected outlier ratio and collection.}

\label{tbl:compare}
\end{table*}

Furthermore, robustness to contamination of \Wa compared to other methods is investigated as shown in Table~\ref{tbl:contamination}.  In most runs, \Wa almost consistently outperforms vis-a-vis other methods on AUC and AUPRC at higher contamination ratio in positive training set and higher ratio of outliers in the mixture.  At lower contamination ratios, \Wa performs similarly with DeepSVDD on AUC and most of the times performs better than ALAD. This shows that \Wa is robust in detecting outliers in various contamination configurations, which makes it use more practical in application, 
settings as real data tends to be contaminated. 
%different from  traditional methods which typically use  non-contaminated data during training. 

%\begin{SCtable*}[][h!]
\begin{table*}[h!]

\centering
\resizebox{.72\textwidth}{!}{

\begin{tabular}[H]{llcrr|rr|rr|rr|rr}
\cmidrule{4-13}

& & & \multicolumn{2}{c|}{\textbf{DeepSVDD \cite{ruff2018deep}} }
& \multicolumn{2}{c|}{\textbf{ALAD \cite{zenati2018adversarially}}} 
& \multicolumn{2}{c|}{\textbf{WAE \cite{tolstikhin2018wasserstein}}} 
& \multicolumn{2}{c|}{\textbf{CoRa \cite{tian2019learning} + Adv.}}
& \multicolumn{2}{c}{\textbf{WALDO}}  \\ 

\cmidrule{4-13}

&Cont.&  $\nu$ & \multicolumn{1}{c}{AUC} & \multicolumn{1}{c|}{AUPRC}  
 &  \multicolumn{1}{c}{AUC}   & \multicolumn{1}{c|}{AUPRC} 
 &  \multicolumn{1}{c}{AUC} & \multicolumn{1}{c|}{AUPRC} 
 &  \multicolumn{1}{c}{AUC}  & \multicolumn{1}{c|}{AUPRC}
 & \multicolumn{1}{c}{AUC} & \multicolumn{1}{c}{AUPRC}  \\ 
\toprule

\multirow{ 8}{*}{\textbf{MNIST}} 
 &\multirow{4}{*}{$5\%$}
& 0.05
& $96.74$ 
& $66.94$
& $66.25  $
& $07.11 $ 
& $97.52$
& $62.43$
& $91.58$ 
& $23.11$
& $\boldsymbol{97.82}$ 
& $\boldsymbol{63.43}$  
\\
 && 0.1
& $96.45$ 
& $74.72$
& $72.09 $ 
& $17.01 $ 
& $94.41$ 
& $67.82$
& $93.34$ 
& $41.17$ 
& $\boldsymbol{98.32}$ 
& $\boldsymbol{86.05}$   
\\
 && 0.2
& $95.87$ 
& $83.69$
& $56.60$ 
& $20.39$ 
& $85.23$ 
& $47.28$
& $94.73$ 
& $69.36$ 
& $\boldsymbol{99.68}$ 
& $\boldsymbol{98.68}$
\\
 && 0.5
& $93.77$ 
& $92.90$
& $50.45$ 
& $47.04$ 
& $69.01$ 
& $70.99$
& $98.18$ 
& $97.89$ 
& $\boldsymbol{99.51}$ 
& $\boldsymbol{99.50}$  \\\cmidrule{2-13}
&\multirow{4}{*}{$10\%$}

& 0.05
& $\boldsymbol{95.89}$ 
& $\boldsymbol{68.53}$
& $89.38$ 
& $16.65$ 
& $84.43$ 
& $23.86$
& $88.99$ 
& $21.15$
& $93.36$ 
& $36.02$  
\\
 && 0.1
& $95.34$ 
& $73.91$
& $72.33$ 
& $17.61$ 
& $80.28$ 
& $33.43$
& $90.63$ 
& $40.34$ 
& $\boldsymbol{97.24}$ 
& $\boldsymbol{82.96}$   
\\
 && 0.2
& $94.95$ 
& $83.96$
& $79.65$ 
& $42.49$ 
& $82.98$ 
& $54.10$
& $92.63$ 
& $61.46$ 
& $\boldsymbol{99.12}$ 
& $\boldsymbol{97.10}$   
\\
 && 0.5
& $93.43$ 
& $93.09$
& $81.44$ 
& $77.11$ 
& $88.40$ 
& $86.53$
& $98.18$ 
& $97.89$ 
& $\boldsymbol{99.51}$ 
& $\boldsymbol{99.50}$
\\ \midrule

\multirow{ 8}{*}{\textbf{FMNIST}}
&\multirow{ 4}{*}{{$5\%$}} 
& 0.05
& $85.66$ 
& $21.31$
& $35.66 $ 
& $04.86 $ 
& $88.12$ 
& $22.56$
& $87.85$ 
& $44.41$ 
& $\boldsymbol{88.55}$ 
& $\boldsymbol{41.18}$
\\
 && 0.1
& $84.82$ 
& $33.86$
& $39.83$ 
& $12.63 $ 
& $88.62$ 
& $39.79$
& $87.82$ 
& $62.54$ 
& $\boldsymbol{90.21}$ 
& $\boldsymbol{66.22}$
\\
 && 0.2
& $83.72$ 
& $49.30$
& $39.73$ 
& $21.83$ 
& $89.24$ 
& $62.36$
& $86.49$ 
& $72.55$ 
& $\boldsymbol{89.71}$ 
& $\boldsymbol{75.57}$
\\
 && 0.5
& $83.07$ 
& $77.66$
& $39.92 $ 
& $48.13 $ 
& $\boldsymbol{88.67}$ 
& $85.44$
& $84.62$ 
& $87.47$ 
& $87.95$ 
& $\boldsymbol{88.92}$\\\cmidrule{2-13}

&\multirow{ 4}{*}{{$10\%$}} 
& 0.05
& $83.89$ 
& $32.98$
& $35.57 $ 
& $04.84$ 
& $84.22$ 
& $17.68$
& $87.85$ 
& $\boldsymbol{44.53}$ 
& $\boldsymbol{88.74}$ 
& $43.10$  
\\
 && 0.1
& $83.68$ 
& $30.72$
& $39.78$ 
& $12.57$ 
& $83.23$ 
& $29.92$
& $87.82$ 
& $62.32$ 
& $\boldsymbol{90.15}$ 
& $\boldsymbol{66.30}$   
\\
 && 0.2
& $82.56$ 
& $43.98$
& $39.71$ 
& $21.83$ 
& $85.03$ 
& $48.22$
& $86.52$ 
& $72.54$ 
& $\boldsymbol{89.53}$ 
& $\boldsymbol{75.80}$   
\\
 && 0.5
& $82.05$ 
& $71.32$
& $39.91 $ 
& $48.11$ 
& $84.87$ 
& $80.80$
& $84.63$ 
& $87.49$
& $\boldsymbol{87.54}$ 
& $\boldsymbol{88.81}$\\\midrule
\multirow{ 8}{*}{\textbf{CIFAR10}}
&\multirow{ 4}{*}{{$5\%$}} 
& 0.05
& $60.89$ 
& $6.42$ 
& $77.14$ 
& $12.12$ 
& $78.88$ 
& $13.15$
& $78.91$ 
& $13.12$ 
& $\boldsymbol{79.04}$ 
& $\boldsymbol{13.19}$  
\\
 && 0.1
& $59.66$ 
& $11.92$
& $71.94$ 
& $21.03$ 
& $74.03$ 
& $23.26$
& $\boldsymbol{74.09}$ 
& $23.20$ 
& $74.02$ 
& $\boldsymbol{23.23}$   
\\
 && 0.2
& $55.78$ 
& $20.77$
& $68.27$ 
& $33.24$ 
& $70.56$ 
& $\boldsymbol{35.39}$
& $70.57$ 
& $35.29$
& $\boldsymbol{70.59}$ 
& $35.37$   
\\
 && 0.5
& $55.46$ 
& $50.23$
& $67.78$ 
& $63.66$ 
& $71.47$ 
& $67.44$
& $\boldsymbol{71.75}$ 
& $\boldsymbol{67.73}$ 
& $71.47$ 
& $67.44$\\\cmidrule{2-13}
&\multirow{ 4}{*}{$10\%$} 
& 0.05
& $55.28$ 
& $5.49$ 
& $77.54 $ 
& $12.42$ 
& $78.87$ 
& $13.14$
& $78.87$ 
& $13.07$ 
& $\boldsymbol{78.97}$ 
& $\boldsymbol{13.18}$
\\
 && 0.1
& $53.17$ 
& $10.04$
& $71.82$ 
& $20.60$ 
& $73.99$ 
& $\boldsymbol{23.25}$
& $73.98$ 
& $23.20$ 
& $\boldsymbol{74.05}$ 
& $\boldsymbol{23.25}$
\\
 && 0.2
& $52.53$ 
& $19.64$
& $68.31$ 
& $33.37$ 
& $70.47$ 
& $35.34$
& $70.40$ 
& $35.18$ 
& $\boldsymbol{70.57}$ 
& $\boldsymbol{35.36}$
\\
 && 0.5
& $51.28$ 
& $49.57$
& $68.02$ 
& $63.63$ 
& $71.45$ 
& $67.44$
& $71.19$ 
& $67.31$ 
& $\boldsymbol{71.5}$ 
& $\boldsymbol{67.46}$\\\midrule

\multirow{ 8}{*}{\textbf{KDD}}
&\multirow{ 4}{*}{$5\%$} 
& 0.05
& $95.18$ 
& $37.19$
& $99.04$ 
& $75.83$ 
& $96.94$ 
& $41.26$
& $97.20$ 
& $45.20$ 
& $\boldsymbol{99.44}$ 
& $\boldsymbol{79.89}$  
\\
 && 0.1
& $95.29$ 
& $55.53$
& $\boldsymbol{99.13}$ 
& $\boldsymbol{86.34}$ 
& $96.78$ 
& $56.87$
& $97.41$ 
& $64.82$ 
& $99.12$ 
& $83.54$   
\\
 && 0.2
& $95.39$ 
& $73.70$
& $99.18$ 
& $93.48$ 
& $96.43$ 
& $71.39$
& $97.40$ 
& $79.92$ 
& $\boldsymbol{99.44}$ 
& $\boldsymbol{95.22}$   
\\
 && 0.5
& $96.06 $
& $91.71$
& $98.74 $
& $91.80 $
& $94.88$
& $87.44$ 
& $97.45$
& $94.66$
& $\boldsymbol{99.16}$
& $\boldsymbol{98.23}$\\\cmidrule{2-13}

&\multirow{ 4}{*}{$10\%$}
& 0.05
& $93.34$ 
& $29.47$
& $\boldsymbol{99.31}$ 
& $\boldsymbol{78.79}$ 
& $95.24$ 
& $31.82$
& $97.22$ 
& $45.17$ 
& $97.71$ 
& $50.65$  
\\
 && 0.1
& $94.35$ 
& $48.75$
& $99.16$ 
& $\boldsymbol{88.96}$ 
& $94.99$ 
& $46.90$
& $97.29$ 
& $62.94$ 
& $\boldsymbol{99.33}$ 
& $87.62$   
\\
 && 0.2
& $94.28$ 
& $67.52$
& $98.61 $ 
& $91.30 $ 
& $94.61$ 
& $63.71$
& $97.05$ 
& $77.17$ 
& $\boldsymbol{98.69}$ 
& $\boldsymbol{87.74}$   
\\
 && 0.5
& $94.74$ 
& $89.37$
& $86.79$ 
& $70.65$ 
& $92.82$ 
& $85.42$
& $96.91$ 
& $92.82$
& $\boldsymbol{98.44}$ 
& $\boldsymbol{96.52}$\\

\bottomrule
\end{tabular}}

\caption{Robustness Evaluation and comparison with state of the art models, showing AUC and AUPRC. The $Cont.$ column represents the percentage of outliers in the positive training data and $\nu$ column represents the 
percentage of outliers in the mixture.
In \textbf{bold}: best result for selected $\nu$, Cont. and collection.}

\label{tbl:contamination}

\end{table*}
%\end{SCtable*}

\subsection{\Wa vs. other methods}
We compare the detection accuracy of \Wa with other methods using AUC and AUPRC metrics. At the outset we caution that direct comparison between the methods is problematic because of the nature of the methods. For example, DeepSVDD~\cite{ring2018flow}, ALAD~\cite{zenati2018adversarially} and even WAE~\cite{tolstikhin2018wasserstein} require a threshold for determining outliers while both CoRa~\cite{tian2019learning} and \Wa do not. For example, there are two versions of  DeepSVDD, the first called {\it soft boundary Deep SVDD} requires a hyperparameter that controls the trade-off on how many data points are allowed to fall outside the hypersphere boundary. The second version scores each point based on the distance from the center of the hypersphere induced - the further the distance the more likely it is an outlier and thus requires a threshold cutoff to label points as outliers. Also we have improved CoRa~\cite{tian2019learning} with \textbf{Advantage} to make it more stable and we compare against this improved version. 

The comparison between the methods is shown in Table ~\ref{tbl:compare}. The $\nu$ column represents the 
percentage of outliers in the mixture. In thirteen out of the sixteen
cases, \Wa does better on the AUPRC metric and in ten out of sixteen it does
better on AUC. On MNIST, both WAE and CoRa have wide confidence intervals while those of \Wa are relatively tight. On Fashion MNIST at $\nu = 0.1$,
DeepSVDD has a slightly higher AUPRC but \Wa has a tighter confidence interval.
In summary we can conclude that \Wa is competitive outlier detection approach
vis-a-vis representative deep learning methods.  An observation worth highlighting is that AUPRC tends to be lower than AUC across the methods. This suggests that there is lot of room for improvement for outlier detection methods in general.
\section{Discussion and Conclusion}
We propose {\bf WALDO}, an extension of deep autoenconders to both detect and generate outliers. \Wa uses the Wasserstein metric distance (WD) to train
an autoencoder which has an inlier and an outlier decoder but a common encoder. The WD is ideally suitable for
outlier detection as it can gracefully handle distributions which may not have identical support. We give an example of detecting extremely rare patterns on a retailer data set. Besides being an accurate
outlier detector, \Wa can be used to generate outliers. This may have potentially widespread application including rare event
simulation and data augmentation. We give one example where outlier generation can be used to create network attacks using
the benchmark KDD99 Cup data set.  Applying \Wa to  application areas like health, transportation and climate change, where outliers are ever present, may yield promising insights. 

\iffalse
Deep Learning makes it possible to infer complex probability distributions from
their samples even when the distributions do not have an analytical form. In the case of outlier detection, we do not have direct access to outlier samples but only  from a mixture of an inlier and outlier distribution. The question we have addressed in the paper, and to the best of our knowledge, the first time it has been asked, 
can we learn an outlier generative distribution?  We have answered in the affirmative 
by designing an autoencoder (\Wa) which has a separate decoder for inliers and outliers but a common encoder. We have used the Wasserstein metric ($W_p$) to design a loss function to train the autoencoder as it has particularly suitable properties for the task in hand including the ability to meaningfully compare probability distributions which have potentially disjoint support. On the benchmark KDD 1999 Cup data set we have shown how the outlier generator can be used to create novel attacks. These can be used for simulations to protect network infrastructure. The paper raises several questions which are worthy of future investigation.  For example, addressing data efficiency questions to simultaneously learn inlier and outlier generators is promising course of study. Applying \Wa to  application areas like health, transportation and climate change, where outliers are de-rigueur and ever present may yield
promising insights. 
\fi
%\input{impact}
{\footnotesize \bibliographystyle{plain}
\bibliography{main2}
}

\clearpage
\onecolumn
\section*{Analysis of WALDO}
We analyze theoretical aspects of 
\Wa for the special case of $p=2$. In particular, we show that under certain circumstances,  $W_2(P^{u}_X,P^{o}_G)$ upper bounds a positive weighted sum of $W_2(P^{o}_X,P^{o}_G)$, $W_2(P^{i}_X,P^{i}_G)$ and $W_2(Q_Z,P_Z)$. Thus by minimizing
an upper bound we can indirectly optimize the decoders. We use the following 
characterization of WAE [20] for decoders with the added assumption that they are  Lipschitz with constant $\gamma$.\\
%\begin{align*}
\vspace{-0.1 cm}
\[W_{2}(P_X,P_G) 
= \inf_{Q}\sqrt{\E_{X \sim P_X}\|X - G(Q(X))\|^{2}} + \gamma.W_{2}(Q_z,P_z)\]

%\end{align*}
\begin{thm}
For a system with an inlier decoder $P^{i}_G$ and an outlier decoder $P^{o}_G$ and a shared deterministic encoder $Q$, there exists positive constants $\alpha,\beta,\delta $ such that

\begin{align*}
W_2(P^{u}_X,P^{o}_G) \geq \alpha W_{2}(P^{o}_X,P^{o}_G) +
\beta W_{2}(P^{i}_X,P^{i}_G) + \delta W_{2}(Q_z,P_z)
\end{align*}
\end{thm}
\begin{proof}
\begin{align*}
    W_2(P^{u}_X,P^{o}_G) = & 
    \inf_{Q}\sqrt{\underset{X \sim P^{u}_{X}}{\E}\|X - G^{o}(Q(X))\|^{2}} + \gamma.W_{2}(Q^{u}_z,P_z) \\
    = & 
    {\scriptstyle \inf_{Q}\sqrt{\nu\underset{X \sim P^{o}_{X}}{\E}\|X - G^{o}(Q(X))\|^{2} +
     (1 - \nu)\underset{X \sim P^{i}_{X}}{\E}\|X - G^{o}(Q(X)\|^{2}}} + \\ & +\gamma.W_{2}(Q_z,P_z) \\
     \geq & \sqrt{\frac{\nu}{2}}\inf_{Q}
      \sqrt{\underset{X \sim P^{o}_{X}}{\E}\|X - G^{o}(Q(X))\|^{2}} + \\
       & \sqrt{\frac{1 -\nu}{2}}\inf_{Q}
      \sqrt{\underset{X \sim P^{i}_{X}}{\E}\|X - G^{o}(Q(X))\|^{2}} +\gamma.W_{2}(Q_z,P_z) \\
     \geq & \sqrt{\frac{\nu}{2}}W_{2}(P^{o}_X,P^{o}_G)
      + \sqrt{\frac{1 -\nu}{2}}W_{2}(P^{i}_X,P^{i}_G) +\\
       & \gamma.\left(1 -\sqrt{\frac{\nu}{2}} -
      \sqrt{\frac{1-\nu}{2}}\right)W_{2}(Q_z,P_z)
\end{align*}
\end{proof}
The proof uses that  (i) $\sqrt{a+b} \geq \frac{1}{\sqrt{2}}(\sqrt{a} + \sqrt{b})$ and (ii)
$\inf(f(x) + g(x))) \geq \inf(f(x)) + \inf(g(x))$ and
that, (iii) the reconstruction error of the optimal decoder $G^{o}(G^{i})$ is the smallest on $P^{o}_X (P^{i}_X)$ and (iii) the network shares common encoder $Q_Z$. A similar bound holds
for  $W_2(P^{u}_X,P^{i}_G)$.

\section*{Supplement for reproducibility}
%\subsection*{Source code}

All Python implementations of the methods evaluated in the paper can be found in Table \ref{tbl:source_codes}. Hyperparameters settings vary for specific collections and methods and can be found in the source code of the implementations.

%\subsection*{Hyperparameters}
%Hyperparameters of the evaluated models are shown in Table \ref{tbl:hyper}.
%Additional details can be found in the source-code.

\subsection{Neural Network architectures}
For \Wa, CoRa and WAE, we use autoencoders without any convolutional layers for KDD, and convolutional autoencoders for image datasets (Cifar-10, Fashion-MNIST, MNIST). A detailed table to show each layer for each datasets can be found in next page.

%======source code 
\begin{table*}
\resizebox{\linewidth}{!}{
\begin{tabular}{l p{15cm} l}
	\toprule
  \textbf{Method} & \textbf{Source code of implementation} & \textbf{Framework} \\ \midrule
  DeepSVDD & \texttt{https://github.com/lukasruff/Deep-SVDD-PyTorch}  & PyTorch\\
  ALAD & \texttt{https://github.com/houssamzenati/Adversarially-Learned-Anomaly-Detection} & TensorFlow \\
  Wasserstein AE & \texttt{https://github.com/sedelmeyer/wasserstein-auto-encoder/blob/master/ Wasserstein-auto-encoder\_tutorial.ipynb} & PyTorch \\
  CoRA (+ advantage) & \texttt{https://github.com/kddblind/waldo} & PyTorch \\
  WALDO & \texttt{https://github.com/kddblind/waldo} & PyTorch \\ \bottomrule
  \end{tabular}}
  \caption{Source code repositories for each method evaluated in this work.}
  \label{tbl:source_codes}
\end{table*}

\begin{table}[h!]
\centering

{
\begin{tabular}{lcr|r|r}
\toprule

& & \multicolumn{1}{c|}{{Non Linearity}}
& \multicolumn{1}{c|}{{Output Shape }} 
& \multicolumn{1}{c}{{Parameters
 }} 
\\

\toprule
\multirow{ 4}{*}{{Encoder}} & 
& Linear-1                                 
& ([-1, 120, 64)                           
& 7,808
\\
&
& LeakyReLU-2                                    
& (-1, 120, 64)                            
& 0

\\
&
&Linear-3                     
& (-1, 120, 32 )                      
& 2,080

\\
&
& ReLU-4                                    
& (-1, 120, 32 )                             
& 0

\\ \midrule
\multirow{ 4}{*}{{Inlier Decoder}} 
& 
& Linear-1             
& (-1, 120, 64)                 
& 2,112

\\
&
& LeakyReLU-2                       
& (-1, 120, 64)                            
& 0

\\
&
& Linear-3                    
& (-1, 120, 121)                   
& 7,865
\\
&
& Tanh-4                      
& (-1, 120, 121)                         
& 0

\\ \midrule
\multirow{ 4}{*}{{Outlier Decoder}} 
& 
& Linear-1             
& (-1, 120, 64)                
& 2,112

\\
&
& LeakyReLU-2                       
& (-1, 120, 64)                           
& 0

\\
&
& Linear-3                    
& (-1, 120, 121)                   
& 7,865
\\
&
& Tanh-4                      
& (-1, 120, 121)                         
& 0

\\ \midrule
\multirow{ 4}{*}{{Discriminator}} 
& 
& Linear-1             
& (-1, 120, 32)                          
&1,056

\\
&
& LeakyReLU-2                        
& (-1, 120, 32)                                         
& 0

\\
&
& Dropout-3                                  
& (-1, 120, 32)                            
& 0
\\
&
& Linear-4                                      
& (-1, 120, 1)                                  
& 33

\\
&
& Sigmoid-5                                    
& (-1, 120, 1)                        
& 0

\\

\bottomrule
\end{tabular}}

\caption{KDD Architecture}
\label{tbl:kdd_architecture}
\end{table}

% %=================CIFAR10

\begin{table}[h!]
\centering
{
\begin{tabular}{lcr|r|r}
\toprule

& & \multicolumn{1}{c|}{{Non Linearity}}
& \multicolumn{1}{c|}{{Output Shape }} 
& \multicolumn{1}{c}{{Parameters
 }} 
\\

\toprule
\multirow{ 4}{*}{{Encoder}} & 
& Conv2d-1                    
& (-1, 32, 16, 16)                           
& 1,536
\\
&
& ReLU-2                                         
&(-1, 32, 16, 16)                                            
& 0
\\
&
& Conv2d-3                                       
&(-1, 64, 8, 8)                                     
& 32,768
\\
&
& BatchNorm2d-4                                            
&(-1, 64, 8, 8)                                             
& 128

\\
&
& ReLU-5                                                      
&(-1, 64, 8, 8)                                            
& 0

\\
&
& Conv2d-6                                                  
&(-1, 128, 4, 4)                                            
& 131,072

\\
&
& BatchNorm2d-7                                             &(-1, 128, 4, 4)                                            
& 256

\\
&
& ReLU-8                                                       
&(-1, 128, 4, 4)                                            
& 0

\\
&
& Conv2d-9                                              
&(-1, 256, 2, 2)                                            
& 524,288

\\
&
& BatchNorm2d-10                                            
&(-1, 256, 2, 2)                                                & 512

\\
&
&  ReLU-11                                                      &(-1, 256, 2, 2)                                                  & 0

\\
&
& Linear-12                                                    &(-1, 32)                                                      
& 32,800

\\ \midrule
\multirow{ 4}{*}{{Inlier Decoder}} 
& 
& Linear-1             
& (-1, 3, 20736)                         
& 684,288

\\
&
& ReLU-2                
& (-1, 3, 20736)                                    
& 0

\\
&
& ConvTranspose2d-3                      
&(-1, 128, 12, 12)                       
& 524,416
\\
&
& BatchNorm2d-4                            
& (-1, 128, 12, 12)                                    
& 256

\\
&
& ReLU-5                                      
& (-1, 128, 12, 12)                                                  
& 0

\\
&
& ConvTranspose2d-6                                     
& (-1, 64, 15, 15)                                             
& 131,136

\\
&
& BatchNorm2d-7                                      
& (-1, 64, 15, 15)                                                
& 128

\\
&
& ReLU-8                                    
& (-1, 64, 15, 15)                                             
& 0

\\
&
& ConvTranspose2d-9                                   
& (-1, 3, 32, 32)                                            
& 3075

\\
&
& Sigmoid-10                                       
&(-1, 3, 32, 32)                                             
& 0

\\ \midrule

\multirow{ 4}{*}{{Outlier Decoder}} 
& 
& Linear-1             
& (-1, 3, 20736)                         
& 684,288

\\
&
& ReLU-2                
& (-1, 3, 20736)                                    
& 0

\\
&
& ConvTranspose2d-3                      
&(-1, 128, 12, 12)                       
& 524,416
\\
&
& BatchNorm2d-4                            
& (-1, 128, 12, 12)                                    
& 256

\\
&
& ReLU-5                                      
& (-1, 128, 12, 12)                                                  
& 0

\\
&
& ConvTranspose2d-6                                     
& (-1, 64, 15, 15)                                             
& 131,136

\\
&
& BatchNorm2d-7                                      
& (-1, 64, 15, 15)                                                
& 128

\\
&
& ReLU-8                                    
& (-1, 64, 15, 15)                                            
& 0

\\
&
& ConvTranspose2d-9                                   
& (-1, 3, 32, 32)                                             
& 3075

\\
&
& Sigmoid-10                                       
&(-1, 3, 32, 32)                                             
& 0

\\ \midrule
\multirow{ 4}{*}{{Discriminator}} 
& 
& Linear-1             
& (-1, 3, 128)                              
&4,224

\\
&
& ReLU-2                                  
& (-1, 3, 128)                                                       
& 0

\\
&
& Linear-3                                                
& (-1, 3, 128)                                 
& 16,512
\\
&
& ReLU-4                                                  
& (-1, 3, 128)                                            
& 0

\\
&
& Linear-5                                                  
& (-1, 3, 128)                             
& 16,512

\\
&
& ReLU-6               
& (-1, 3, 128)                                    
& 0

\\
&
& Linear-7                                                               
& (-1, 3, 128)                                      
& 16,512

\\
&
& ReLU-8                                                               
& (-1, 3, 128)                                           
&0

\\
&
& Linear-9                                                                 
& (-1, 3, 1)                                      
& 129

\\
&
& Sigmoid-10                                                    
&(-1, 3, 1)               
& 0

\\ 

\bottomrule
\end{tabular}}
\caption{CIFAR10 Architecture}
\label{tbl:cifar_architecture}
\end{table}

%=================MNIST

\begin{table}[ht!]
\centering
{
\begin{tabular}{lcr|r|r}
\toprule

& & \multicolumn{1}{c|}{{Non Linearity}}
& \multicolumn{1}{c|}{{Output Shape }} 
& \multicolumn{1}{c}{{Parameters
 }} 
\\

\toprule
\multirow{ 4}{*}{{Encoder}} & 
& Conv2d-1           
& (-1, 40, 14, 14)             
& 640
\\
&
& Conv2d-3                      
& (-1, 80, 7, 7)                      
& 51,200
\\
&
& BatchNorm2d-4                  
&(-1, 80, 7, 7)                    
& 160
\\
&
& ReLU-5                      
& (-1, 80, 7, 7)                         
& 0
\\
&
& Conv2d-6                              
& (-1, 160, 3, 3)                                
& 204,800
\\
&
& BatchNorm2d-7                             
& (-1, 160, 3, 3)                                   
& 320

\\
&
& ReLU-8                                     
& (-1, 160, 3, 3)                                      
& 0

\\
&
& Conv2d-9                                        
& (-1, 320, 1, 1)                                            
& 819,200

\\
&
& BatchNorm2d-10                                      
& (-1, 320, 1, 1)                                     
& 640
\\
&
& ReLU-11                                         
& (-1, 320, 1, 1)                                               & 0
\\
&
& Linear-12                                              
&(-1, 32)                                        
& 10,272

\\ \midrule
\multirow{ 4}{*}{{Inlier Decoder}} 
& 
& Linear-1             
& (-1, 3, 15680)         
& 517,440

\\
&
& ReLU-2             
& (-1, 3, 15680)               
& 0

\\
&
& ConvTranspose2d-3          
& (-1, 160, 10, 10)         
& 819,360
\\
&
& BatchNorm2d-4          
& (-1, 160, 10, 10)            
& 320

\\
&
& ReLU-5          
& (-1, 160, 10, 10)               
& 0
\\
&
& ConvTranspose2d-6           
& (-1, 80, 13, 13)        
& 204,880

\\
&
& BatchNorm2d-7           
& (-1, 80, 13, 13)            
& 160

\\
&
& ReLU-8           
& (-1, 80, 13, 13)              
&0

\\
&
&ConvTranspose2d-9            
& (-1, 1, 28, 28)          
& 1,281
\\
&
&Sigmoid-10                       
& (-1, 1, 28, 28)          
& 0
\\ \midrule
\multirow{ 4}{*}{{Outlier Decoder}} 
& 
& Linear-1             
& (-1, 3, 15680)         
& 517,440

\\
&
& ReLU-2             
& (-1, 3, 15680)               
& 0

\\
&
& ConvTranspose2d-3          
& (-1, 160, 10, 10)         
& 819,360
\\
&
& BatchNorm2d-4          
& (-1, 160, 10, 10)            
& 320

\\
&
& ReLU-5          
& (-1, 160, 10, 10)             
& 0
\\
&
& ConvTranspose2d-6           
& (-1, 80, 13, 13)       
& 204,880

\\
&
& BatchNorm2d-7           
& (-1, 80, 13, 13)            
& 160

\\
&
& ReLU-8           
& (-1, 80, 13, 13)               
&0

\\
&
&ConvTranspose2d-9            
& (-1, 1, 28, 28)           
& 1,281
\\
&
&Sigmoid-10                       
& (-1, 1, 28, 28)          
& 0
\\ \midrule

\multirow{ 4}{*}{{Discriminator}} 
& 
& Linear-1             
& (-1, 3, 160)                 
& 5,280

\\
&
& ReLU-2             
& (-1, 3, 160)                            
& 0

\\
&
& Linear-3                     
& (-1, 3, 160)                
& 25,760
\\
&
& ReLU-4                        
& (-1, 3, 160)                      
& 0

\\
&
& Linear-5                      
& (-1, 160, 10, 10)              
& 25,760
\\
&
& ReLU-6                         
& (-1, 3, 160)                   
& 0

\\
&
& Linear-7                       
& (-1, 3, 160)                      
& 25,760

\\
&
& ReLU-8           
& (-1, 3, 160)                    
&0

\\
&
&Linear-9                            
& (-1, 3, 1)                 
&161
\\
&
&Sigmoid-10                       
& (-1, 3, 1)                         
& 0
\\
\bottomrule
\end{tabular}}
\caption{MNIST \& FMNIST Architecture}
\label{tbl:mnist_architecture}
\end{table}

\end{document}